\documentclass[sigconf]{aamas}

\renewcommand\footnotetextcopyrightpermission[1]{} 
\usepackage{amsmath,amssymb,amsthm}
\usepackage{subcaption}
\usepackage{placeins}

\usepackage{algorithm}
\usepackage[noend]{algpseudocode}

\usepackage[textwidth=2.5cm,textsize=footnotesize,disable]{todonotes}
\usepackage{pgf}
\usepackage{tikz}
\usetikzlibrary{arrows,automata}
\usepackage[latin1]{inputenc}

\makeatletter
\def\@copyrightspace{\relax}
\makeatother

\begin{document}


\title{Combining No-regret and Q-learning}
\titlenote{Part of this work was done while Ian Kash was at Microsoft Research. We gratefully acknowledge support from the National Science Foundation via award CCF-1934915. This represents the full version of the paper from AAMAS 2020.}
\author{Ian A. Kash}
\affiliation{%
 \institution{University of Illinois at Chicago}
 \city{Chicago} 
 \state{IL} 
}
\email{iankash@uic.edu}

\author{Michael Sullins}
\affiliation{%
 \institution{University of Illinois at Chicago}
 \city{Chicago} 
 \state{IL} 
}
\email{sullins2@uic.edu}

\author{Katja Hofmann}
\affiliation{%
 \institution{Microsoft Research}
 \city{Cambridge} 
 \state{UK} 
}
\email{katja.hofmann@microsoft.com}

\keywords{No-regret learning; reinforcement learning; CFR}

\begin{abstract}
Counterfactual Regret Minimization (CFR) has found success in settings like poker which have both terminal states and perfect recall.  We seek to understand how to relax these requirements.  As a first step, we introduce a simple algorithm, local no-regret learning (LONR), which uses a Q-learning-like update rule to allow learning without terminal states or perfect recall.  We prove its convergence for the basic case of MDPs (where Q-learning already suffices), as well as limited extensions of them. With a straightforward modification, we extend the basic premise of LONR to work in multi-agent settings and present empirical results showing that, with appropriate parameterization, it achieves last iterate convergence in a number of settings. Most notably, we show this for NoSDE games, a class of Markov games specifically designed to be impossible for Q-value-based methods to learn and  where no prior algorithm is known to achieve convergence to a stationary equilibrium even on average.

\end{abstract}

\maketitle


\section{Introduction}


Versions of counterfactual regret minimization (CFR) \citep{zinkevich2008regret} have found success in playing poker at human expert level \citep{brown2019superhuman,moravvcik2017deepstack} as well as fully solving non-trivial versions of it~\citep{bowling2015heads}.  CFR more generally can solve extensive form games of incomplete information.  It works by using a no-regret algorithm to select actions.  In particular, one copy of such an algorithm is used at each {\em information set}, which corresponds to the full history of play observed by a single agent.  The resulting algorithm satisfies a global no-regret guarantee, so at least in two-player zero-sum games is guaranteed to converge to an optimal strategy through sufficient self-play.

However, CFR does have limitations.  It makes two strong assumptions which are natural for games such as poker, but limit applicability to further settings.  First, it assumes that the agent has perfect recall, which in a more general context means that the state representation captures the full history of states visited
(and so imposes a tree structure).
Current RL domains may rarely repeat states due to their large state spaces, but they certainly do not encode the full history of states and actions.
Second, it assumes that a terminal state is eventually reached and performs updates only after this occurs. 
Even in episodic RL settings, which do have terminals, it may take thousands of steps to reach them.
Neither of these assumptions is required for traditional planning algorithms like value iteration or reinforcement learning algorithms like Q-learning.  Nevertheless, approaches inspired by CFR have shown empirical promise in domains that do not necessarily satisfy these requirements~\citep{jin2017regret}.

In this paper, we take a step toward relaxing these assumptions.  We develop a new algorithm, which we call local no-regret learning (LONR). In the same spirit as CFR, LONR uses a copy of an arbitrary no-regret algorithm in each state. (For technical reasons we require a slightly stronger property we term no-absolute-regret.)
The updates for these algorithms are computed in the style of Q-values, which eliminates the need for perfect recall or terminals.
Our main result is that LONR has the same asymptotic convergence guarantee as value iteration for discounted-reward Markov Decision Processes (MDP).  Our result also generalizes to settings where, from a single agent's perspective, the transition process is time invariant but rewards are not.  Such settings are traditionally interpreted as ``online MDPs''~\citep{even2009online,mannor2003empirical,yu2009markov,ma2015online}, but also include normal form games. We view this as a proof-of-concept for achieving CFR-style results without requiring perfect recall or terminal states.  Under stylized assumptions, we can extend this to asynchronous value iteration and (with a weaker convergence guarantee) a version of RL.

LONR is not an improvement over traditional RL algorithms for solving MDPs.  However, naively applying single-agent RL algorithms in settings with multiple agents, such as Markov games, is known to fail to achieve good performance in many cases~\cite{hu2003nash,zinkevich2006cyclic}.  In contrast, we believe the robustness provided by no-regret learning will more naturally extend beyond MDPs.

To demonstrate this, in our experimental results we explore settings beyond the exact reach of our theoretical results. 
Our main results are on a particular class of Markov games known as NoSDE Markov games, which are specifically designed to be challenging for learning algorithms~\citep{zinkevich2006cyclic}.  These are finite two agent Markov games with no terminal states where No Stationary Deterministic Equilibria exist: all stationary equilibria are randomized.  Worse, by construction Q-values do not suffice to determine the correct equilibrium randomization.  Thus, prior work has focused on designing multiagent learning algorithms which can converge to non-stationary equilibria ~\citep{zinkevich2006cyclic}.  The sorts of cyclic behavior that NoSDE games induce has also been observed in more realistic settings of competition between agents \citep{tesauro2002pricing}.

In contrast, we demonstrate that LONR converges to the stationary equilibrium for specific choices of regret minimizer.  Furthmore, for these choices of minimizer we achieve not just convergence of the average policy but also of the current policy, or last iterate.  Thus our results are also interesting
as they highlight
a setting for the study of last iterate convergence, an area of current interest, in between simple normal form games~\cite{mertikopoulos2018cycles,bailey2018multiplicative} and rich, complex settings such as generative adverarial networks (GANs)~\cite{daskalakis2017training}.



Most work on CFR uses some version of regret matching as the regret minimzer.  However, all prior variants of regret matching are known to not possess last iterate convergence in normal form games such as matching pennies and rock-paper-scissors.  As part of our analysis we introduce a simple novel variant of regret matching that, despite not actually being no-regret\footnote{A preliminary version of this paper erroneously claimed that this algorithm was no-regret.  We thank Dustin Morrill and colleagues for pointing out the error~\cite{morrill2021efficient}.}, results in empirical last iterate convergence in these normal form games as well as other settings we have tried.

\section{Related work}

CFR algorithms remain an active topic of research; recent work has shown how to combine it with function approximation~\citep{waugh2015solving,moravvcik2017deepstack,jin2017regret,brown2018deep,li2018double}, improve the convergence rate in certain settings~\citep{farina2019stable}, and apply it to more complex structures~\citep{farina2018composability}.
Most relevant to our work, examples are known where CFR fails to converge to the correct policy without perfect recall~\citep{lanctot2012no}.

Both CFR and LONR are guaranteed to converge only in terms of their average policy.  This is part of a general phenomenon for no-regret learning in games, where the ``last iterate,'' or current policy, not only fails to converge but behaves in an extreme and cyclic way~\citep{mertikopoulos2018cycles,bailey2018multiplicative,cheung2019vortices,bailey2019finite}.  Recent work has explored cases where it is nonetheless effective to use the last iterate.
In some poker settings a variant of CFR known as CFR+~\citep{tammelin2014solving,bowling2015heads} has good last iterates, but it is known to cycle in normal-form games.
Motivated by training Generative Adversarial Networks (GANs), recent results have shown that certain no-regret algorithms converge in terms of the last iterate to saddle-points in convex-concave min-max optimization problems~\citep{daskalakis2017training,daskalakis2018last}.  The ability to use the last iterate is particularly important in the context of function approximation~\cite{heinrich2016deep,abernethy2019last}.  Our experimental results provide examples of LONR achieving last iterate convergence when the underlying regret minimizer is capable of it.

Prior work has developed algorithms which combine no-regret and reinforcement learning, but in ways that are qualitatively different from LONR.
A common approach in the literature on multi-agent learning is to use no-regret learning as an outer loop to optimize over the space of policies, with the assumption that the inner loop of evaluating a policy is given to the algorithm.  There is a large literature on this approach in normal form games~\citep{greenwald2003general}, where policy evaluation is trivial, and a smaller one on ``online MDPs''~\citep{even2009online,mannor2003empirical,yu2009markov,ma2015online}, where it is less so.  Of particular note in this literature,~\citet{even2005experts} also use the idea of having a copy of a no-regret algorithm for each state.
An alternate approach to solving multi-agent MDPs is to use Q-learning as an outer loop with some other algorithm as an inner loop to determine the collective action chosen in the next state~\citep{littman1994markov,hu2003nash,greenwald2003correlated}. Of particular note, \citet{gondek2004qnr} proposed the use of no-regret algorithms as an inner loop with Q-learning as an outer loop while \citet{even2002pac} use multi-armed bandit algorithms as the inner loop with Phased Q-learning~\citep{kearns1999finite} as the outer loop.
In contrast to these literatures, we combine RL in each step of the learning process rather than having one as an inner loop and the other as an outer loop.

Recent work has drawn new connections between no-regret and RL.  \citet{srinivasan2018actor} show that actor-critic methods can be interpreted as a form of regret minimization, but only analyze their performance in games with perfect recall and terminal states.  This is complementary to our approach, which focuses on value-iteration-style algorithms, in that it suggests a way of extending our results to other classes of algorithms. \citet{neu2017unified} study entropy-regularized RL and interpret it as an approximate version of Mirror Descent, from which no-regret algorithms can be derived as particular instantiations. 
\citet{kovavrik2018analysis} study algorithms that instantiate a regret minimizer at each state without the counterfactual weightings from CFR, but explicitly exclude settings without terminals and perfect recall from their analysis.
\citet{jin2018q} showed that in finite-horizon MDPs, Q-learning with UCB exploration achieves near-optimal regret bounds.

The closest technical approach to that used in our theoretical results is that of \citet{bellemare2016increasing} who introduce new variants of the Q-learning operator.  However, our algorithm is not an operator as the policy used to select actions changes from round to round in a history-dependent way, so we instead directly analyze the sequences of Q-values.



%

\section{Preliminaries}

Consider a Markov Decision Process $M = (\mathcal{S}, \mathcal{A}, P, r, \gamma)$, where $\mathcal{S}$ is the state space, $\mathcal{A}$ is the (finite) action space, $P : \mathcal{S} \times \mathcal{A} \rightarrow \Delta(\mathcal{S})$ is the transition probability kernel, $r : \mathcal{S} \times \mathcal{A} \rightarrow \mathbb{R}$ is the (expected) reward function (assumed to be bounded), and $0 < \gamma < 1$ is the discount rate.  \mbox{(Q-)}value iteration is an operator $\mathcal{T}$, whose domain is bounded real-valued functions over $\mathcal{S} \times \mathcal{A}$, defined as
\begin{equation}
\mathcal{T}Q(s,a) = r(s,a) + \gamma \mathbb{E}_P [\max_{a' \in \mathcal{A}}Q(s',a')]
\end{equation}
Due to the presence of $\gamma$, this operator is a contraction map in $|| \cdot ||_\infty$, and so converges to a unique fixed point $Q^*$, where $Q^*(s,a)$ gives the expected value of the MDP starting from state $s$, taking action $a$, and thereafter following the optimal policy $\pi^*(s) = \arg\max_{a \in \mathcal{A}} Q^*(s,a)$ 
~\citep{bertsekas1996neuro}.

Our algorithm makes use of a no-regret learning algorithm.\footnote{It may seem strange to use an algorithm designed for non-stationary environments in a stationary one.  We do so with the goal of designing an algorithm that generalizes to non-stationary settings such as ``online'' MDPs and Markov games.}  Consider the following (adversarial full-information) setting.  There are $n$ actions $a_1, \ldots a_n$.  At each timestep $k$ an online algorithm chooses a probability distribution $\pi_k$ over the $n$ actions.  Then an adversary chooses a reward $x_{k,i}$ for each action $i$ from some closed interval, e.g. $[0,1]$, which the algorithm then observes.  The (external) regret of the algorithm at time $k$ is
\begin{equation}
\label{eqn:regret}
\frac{1}{k+1}\max_i \sum_{t = 0}^k x_{t,i} - \pi_t \cdot x_t
\end{equation}
An algorithm is {\em no-regret} if there exists a sequence of constants $\rho_k$ such that regardless of the adversary the regret at time $k$ is at most $\rho_k$ and $\lim_{k \rightarrow \infty} \rho_k = 0$.  A common bound is that $\rho_k$ is $O(1/\sqrt{k})$.

For our results, we make use of a stronger property, that the {\em absolute value} of the regret is bounded by $\rho_k$.  We call such an algorithm a {\em no-absolute-regret} algorithm.  Algorithms exist that satisfy the even stronger property that the regret is at most $\rho_k$ and at least 0.  Such {\em non-negative-regret} algorithms include all linear cost Regularized Follow the Leader algorithms, which includes Randomized Weighted Majority and linear cost Online Gradient Descent~\citep{gofer2016lower}.

\section{Local no-regret learning (LONR)}
\label{sec:LONR}

The idea of LONR is to fuse the essence of value iteration / Q-learning and CFR.  A standard analysis of value iteration proceeds by analyzing the sequence of matrices $Q,\mathcal{T}Q, \mathcal{T}^2Q, \mathcal{T}^3Q, \ldots$.  The essence of CFR is to choose the policy for each state locally using a no-regret algorithm.  While doing so does not yield an operator, as the policy changes each round in a history-dependent way, this process still yields a sequence of $Q$ matrices as follows.

Fix a matrix $Q_0$.  Initialize $|\mathcal{S}|$ copies of a no-absolute-regret algorithm (one for each state) with $n = |\mathcal{A}|$ and find the initial policy $\pi_0(s)$ for each state $s$.  Then we iteratively reveal rewards to the copy of the algorithm for state $s$ as $x^s_{k,i} = Q_k(s,a_i)$,\footnote{Note that we are revealing the rewards of {\em all} actions, so we are in the planning setting rather than the standard RL one.  We address settings with limited feedback in Section~\ref{sec:async_bandit}.} and update the policy $\pi_{k+1}$ according to the no-absolute-regret algorithm and $Q_{k+1}(s,a) = r(s,a) + \gamma\mathbb{E}_{P,\pi_k} [Q_k(s',a')]$.

Call this process local no-regret learning (LONR).  It can be viewed as a synchronous version of Expected SARSA~\citep{van2009theoretical} where instead of using an $\epsilon$-greedy policy with decaying $\epsilon$, a no-absolute-regret policy is used instead.  In the rest of this section we work up to our main result, that LONR converges to $Q^*$.  Like many prior results using no-regret learning (e.g. \cite{zinkevich2008regret}), the convergence is of the average of the $Q_k$ matrices.

We work up to this result through a series of lemmas.
To begin, we derive a bound on the average of $Q$ values using the no-absolute-regret property.  We use two slightly different averages to be able to relate them using the $\mathcal{T}$ operator.

\begin{lemma}
\label{lem:regret}
Let $\overline{Q}_k = 1/k\sum_{t = 1}^k Q_{t}$ and $\underline{Q}_k = 1/k\sum_{t = 0}^{k-1} Q_{t}$.
Then
\begin{equation}
\label{eqn:upper}
- \gamma \rho_{k-1}+  \mathcal{T}\underline{Q}_k(s,a) \leq\overline{Q}_k(s,a) \leq  \gamma \rho_{k-1}+  \mathcal{T}\underline{Q}_k(s,a).
\end{equation}
\end{lemma}

\begin{proof}
By the definitions of LONR and no-regret, 
\begin{align*}
\overline{Q}_k(s,a)
&=  \frac{1}{k} \sum_{t = 1}^k Q_{t}(s,a)\\
&=  \frac{1}{k} \sum_{t = 0}^{k-1} r(s,a) + \gamma\mathbb{E}_{P,\pi_{t}}[ Q_{t}(s',a')]\\
&=  r(s,a) +\gamma\mathbb{E}_P [\frac{1}{k} \sum_{t = 0}^{k-1}\mathbb{E}_{\pi_{t}} [Q_{t}(s',a')]]\\
&\geq  r(s,a) + \gamma \mathbb{E}_P [\max_i \frac{1}{k} \sum_{t = 0}^{k-1}Q_{t}(s',a_i) - \rho_{k-1}]\\
&= - \gamma \rho_{k-1} + r(s,a) + \gamma\mathbb{E}_P [\max_i \frac{1}{k} \sum_{t = 0}^{k-1} Q_{t}(s',a_i)]\\
&= - \gamma \rho_{k-1} + r(s,a) + \gamma\mathbb{E}_P [\max_i \underline{Q}_k(s',a_i)]\\
&= - \gamma \rho_{k-1}+  \mathcal{T}\underline{Q}_k(s,a)\\
\end{align*}

The key step is the inequality in the fourth line, where we use the fact that the policy for state $s'$ is being determined by a no-regret algorithm, so we can use Equation~\eqref{eqn:regret} to bound the expected value of the policy by the value of the hindsight-optimal action and the regret bound of the algorithm.
Similarly, by the stronger no-absolute-regret property, we can reverse the inequality to get 
$\overline{Q}_k(s,a) \leq \gamma \rho_{k-1}+  \mathcal{T}\underline{Q}_k(s,a)$.
This proves Equation~\eqref{eqn:upper}.
\end{proof}

Next, we show that the range that the $Q$ values take on is bounded.  This lemma is similar in spirit to Lemma 2 of~\citet{bellemare2016increasing}. The full proof is in Appendix \ref{sec:om-proofs}.

\begin{lemma}
\label{lem:range}
Let $||r||_\infty = \max_{s,a} |r(s,a)|$.  Then $||Q_k - Q_0||_\infty \leq 1 / (1 - \gamma) ||r||_\infty + 2 ||Q_0||_\infty$
\end{lemma}

Combining these two lemmas, we can show that $\underline{Q}_k$ is an approximate fixed-point of $\mathcal{T}$, and that the approximation is converging to 0 as $k \rightarrow \infty$.

\begin{lemma}
\label{lem:contraction}
$||\underline{Q}_k - \mathcal{T}\underline{Q}_k||_\infty \leq \frac{1}{k}(1 / (1 - \gamma) ||r||_\infty + 2 ||Q_0||_\infty) + \gamma \rho_{k-1}$
\end{lemma}

It remains to show that a converging sequence of approximate fixed points converges to $Q^*$, the fixed point of $\mathcal{T}$.

\begin{lemma}
\label{lem:approximateFP}
Let $Q_0,Q_1,\ldots$ be a sequence such that $\lim_{k \rightarrow \infty} ||Q_k - \mathcal{T} Q_k||_\infty = 0$.  Then $\lim_{k \rightarrow \infty} Q_k = Q^*$.
\end{lemma}

Combining Lemmas~\ref{lem:contraction} and~\ref{lem:approximateFP} shows the convergence of LONR learning.

\begin{theorem} \label{thm:main}
$\lim_{k \rightarrow \infty} \underline{Q}_k = Q^*$.
\end{theorem}

\subsection{Beyond MDPs}

While our results do not rely on perfect recall or terminal states the way CFR does, so far they are limited to the case of MDPs while CFR permits multiple agents and imperfect information.  We can straightforwardly extend our results to some settings beyond MDPs.  In Appendix~\ref{sec:beyond} we show that a version of Lemma~\ref{lem:regret} holds in MDP-like settings where the transition probability kernel does not change from round to round but the rewards do.  Examples of such settings include ``online MDPs'' and normal-form games.  This last result is not particularly surprising as with a single state LONR reduces to standard no-regret learning, whose convergence guarantees in normal-form games are well understood.  In Section~\ref{sec:experiments} we present empirical results that, despite a lack of supporting theory, demonstrate convergence in the richer multi-agent setting of Markov games.

\section{Extensions}
\label{sec:extensions}

In this section we consider two extensions to LONR, one allowing it to be updated asynchronously (i.e. not updating every state in every iteration) and the other allowing it to learn from asynchronous updates with bandit feedback (i.e. the standard off-policy RL setting).  These are important as a step toward applying LONR beyond settings small enough for tabular approaches.  This introduces novel technical issues around the performance of no-regret algorithms when their performance is assessed on a random sample of their rounds (rather than all of them).  Therefore, we analyze convergence only in the simplified case where the state to update at each iteration is chosen uniformly at random.  We emphasize that this is an unreasonably strong assumption in practice, and view our results in this section as providing intuition about why sufficiently ``nice'' processes should converge.  We demonstrate empirical convergence in a more standard on-policy setting in Section~\ref{sec:experiments} and leave a more general theoretical analysis to future work.


\subsection{Asynchronous updates}

In Section~\ref{sec:LONR} we analyzed an algorithm, LONR, which is similar to value iteration in that each state is updated synchronously at each iteration.  However, an alternative is to update them asynchronously, where an arbitrary single state is updated at each iteration.  Subject to suitable conditions on the frequency with which each state is updated, asynchronous value iteration also converges~\cite{bertsekas1982distributed}.

 A line of work has shown that CFR will also converge when sampling trajectories~\citep{lanctot2009monte,gibson2012generalized,johanson2012efficient}.

In this section, we show that LONR also converges with such asynchonous updates.  However, this introduces a new complexity to our analysis.  In particular, with synchronous updates there is a guarantee that $\overline{Q}_k(s,a)$ sees exactly the first $k$ values of each action of each of its successor states.  This allows us to immediately apply the no-regret property~\eqref{eqn:regret}.  With asynchronous updates, even if we update all actions in a state at the same time, $\overline{Q}_k(s,a)$'s successors may have been updated more or fewer than $k$ times, and $\overline{Q}_k(s,a)$ may have missed some of these updates and observed others more than once, meaning we cannot directly apply~\eqref{eqn:regret}.  We prove the following Lemma to show that a particular sampling process converges to a correct estimate of the average regret, but believe that similar characterizations should hold for other ``nice'' processes.  We demonstrate empirical convergence of asynchronous LONR when states are selected in an on-policy manner in Section~\ref{sec:experiments}.

\begin{lemma}
\label{lem:bootstrap}
Let $t_1,\ldots,t_k$ be the first $k$ iterations at which $s$ is updated, $s'$ be a successor of $s$, $\tau_1,\ldots,\tau_{k'}$ be the iterations before $t_k$ at which $s'$ was updated, and $\xi_{ss'}(k) = 1/k \sum_{i=1}^k \mathbb{E}_{\pi_{t_i}}Q_{t_i}(s',a) - 1/k' \sum_{i=1}^{k'} \mathbb{E}_{\pi_{\tau_i}}Q_{\tau_i}(s',a)$.
If the state to be updated at each iteration is chosen uniformly at random then $\lim_{k \rightarrow \infty} \xi_{ss'}(k) = 0$ with probability 1.
\end{lemma}

The proof has two main steps: (1) showing that as time grows large the average of the number of times each update to $s'$ is sampled by an update to $s$ goes to 1 and (2) applying a prior result to conclude that this means the average of the samples converges to the true average.

\begin{proof}
Let $X_i$ be the number of times $s$ is updated using $\tau_i$.  The $X_i$ are i.i.d. random variables whose law is the geometric distribution with probability 0.5.  Thus, $\mathbb{E}[X_i] = 1$ and by the strong law of large numbers the sample average of the $X_i$ converges to 1 with probability 1.  Let $c_i = \mathbb{E}_{\pi_{\tau_i}}Q_{\tau_i}(s',a)$ and $C_i = \sum_{i=1}^{k'} c_i$.  Then by \cite[Theorem 3]{etemadi2006convergence}, $\sum_{i = 1}^{k'} c_iX_i / C_i$ also converges to 1 with probability 1.  Equivalently, $\lim_{k' \rightarrow \infty} \sum_{i=1}^{k'} c_iX_i - C_i = 0$ with probability 1.
\end{proof}

With this in hand, we can now prove a result similar to Lemma~\ref{lem:regret} for asynchronous updates.  The primary difference is that now have an additional error term in the bounds, but like the term from the regret it goes to zero per Lemma~\ref{lem:bootstrap}. The full proof is in Appendix \ref{sec:om-proofs}.

\begin{lemma}
\label{lem:regret-async}
Let $s$ be the state selected uniformly at random and updated in iteration $t+1$, for which this is the $k$-th update and let $\overline{Q}_{t+1}(s,a) = 1/k \sum_{i = 1}^k Q_{t_i}(s,a)$ and $\overline{Q}_{t+1}(s',a) = \overline{Q}_{t}(s',a)$ for $s' \neq s$.
Then
\begin{align}
&\min_{s'} \gamma (-\xi_{ss'}(k)- \rho_{k'}) +  \mathcal{T}\overline{Q}_t(s,a) \notag\\ &\leq\overline{Q}_{t+1}(s,a) \leq  \max_{s'} \gamma (-\xi_{ss'}(k) + \rho_{k'}) +  \mathcal{T}\overline{Q}_t(s,a). \label{eqn:upper-async}
\end{align}
\end{lemma}



It immediately follows that $\overline{Q}_t$ is an approximate fixed-point of $\mathcal{T}$, and that the approximation is converging to 0 as $k \rightarrow \infty$.

\begin{lemma}
\label{lem:contraction-async}
Let $k$ be the minimum number of times a state has been chosen uniformly at random for update by time $t$.  Then
$||\overline{Q}_t - \mathcal{T}\overline{Q}_t||_\infty \leq \gamma \rho_{k-1} + ||\xi(k)||_\infty$
\end{lemma}


Combining Lemmas~\ref{lem:contraction-async} and~\ref{lem:approximateFP} (the latter of which applies without change) shows the convergence of asynchronous LONR learning.

\begin{theorem}
If states are chosen for update uniformly at random $\lim_{k \rightarrow \infty} \overline{Q}_t = Q^*$ with prob. 1.
\end{theorem}

 \subsection{Asynchronous updates with bandit feedback}
 \label{sec:async_bandit}

In RL, algorithms like Q-learning are usually assumed not to know $P$ and so only have access to feedback corresponding to the action actually taken in the current iteration.  In such settings, ordinary no-regret algorithms are not applicable because they require the counterfactual results from actions not chosen.  
However, multi-armed bandit algorithms, such as
Exp3~\citep{auer2002nonstochastic}, are designed to achieve no-regret guarantees {\em in expectation} despite only receiving feedback about the outcomes chosen.   It would be natural to adapt LONR to the on-policy RL setting by replacing the no-regret algorithm with a multi-armed bandit one.  This type of result has previously been obtained for normal-form games~\cite{banerjee2005efficient}, where agents can learn to play optimally even if they only learn their payoff at each stage and not what action the other agents took.

To adapt LONR to make use of multi-armed bandit algorithms, we can use the $Q$ update rule $Q_{t+1}(s,a) = 1/\pi_t(s,a) (r(s,a) + \gamma \mathbb{E}_{\pi_t}[Q_k(s',a')]$ if $a$ is the action chosen for state $s$ and $Q_{t+1}(s,a') = 0$ for $a' \neq a$.\footnote{The use of importance sampling here is to maintain the structure that successor states are evaluated as $\mathbb{E}_{\pi_t}[Q_k(s',a')]$.  Alternatively we could use the SARSA-style update $Q_{t+1}(s,a) = r(s,a) + \gamma Q_t(s',a')$ where $a'$ is the action that was chosen the last time $s'$ was updated and leave all other Q-values unchanged (this also requires appropriately adjusting the way the average is computed).}  The no-absolute-regret algorithm for bandit feedback at $s$ can then be updated as $x^s_t = r(s,a) + \gamma \mathbb{E}_{\pi_t}[Q_k(s',a')]$.  (We use the raw rather than importance sampling estimate here because, e.g. Exp3 already includes importance weighting.)  Unlike in Q-learning, we do not need to average over Q-values to account for the stochasticity in choice of $s'$ because our convergence results are already for the averages of our Q-values.

With these definitions, Lemma~\ref{lem:regret-async} can be immediately adapted to this setting with the caveat that now the guarantees only hold in expectation over the choice of action at each iteration and the resulting state.  Furthermore, since we require the state be chosen uniformly at random, the resulting algorithm is on-policy in the sense that the algorithm is choosing which action to receive feedback about, but does not control the sequence of states in which it acts.

\begin{lemma}
\label{lem:regret-bandit}
Let $s$ be the state selected uniformly at random and updated in iteration $t+1$, for which this is the $k$-th update and let $\overline{Q}_{t+1}(s,a) = 1/k \sum_{i = 1}^k Q_{t_i}(s,a)$ and $\overline{Q}_{t+1}(s',a) = \overline{Q}_{t}(s',a)$ for $s' \neq s$.
Then
\begin{align}
&\min_{s'} \gamma (-\xi_{ss'}(k)- \rho_{k'}) +  \mathcal{T}\overline{Q}_t(s,a) \label{eqn:upper-bandit}\\ &\leq \mathbb{E}[\overline{Q}_{t+1}(s,a)] \leq  \max_{s'} \gamma (-\xi_{ss'}(k) + \rho_{k'}) \mathcal{T}\overline{Q}_t(s,a).\notag 
\end{align}
\end{lemma}

The same analysis from the asynchronous full information case then yields the following theorem.


 \begin{theorem}
 If states are chosen for update uniformly at random, then $\lim_{k \rightarrow \infty} E[\overline{Q}_t] = Q^*$.
 \end{theorem}

This convergence of expectation implies that the $\overline{Q}_t$ converge in probability to $Q^*$, a weaker guarantee than the almost sure convergence of algorithms like Q-learning.  We leave deriving a stronger convergence guarantee with more natural assumptions about state selection to future work.

\section{Experiments}
\label{sec:experiments}

Our theoretical results in Sections~\ref{sec:LONR} and~\ref{sec:extensions} are restricted to (online) MDPs and normal form games and require a number of technical assumptions. The primary goal of this section is to provide evidence that relaxation of these restrictions may be possible.

Another goal of these results is that while the theory behind LONR calls for a regret minimizer with the no-absolute regret property, we seek to understand the performance of various well-known regret minimizers within the LONR framework, which may or may not be no-absolute regret.
One popular class of no-regret algorithms is Follow-the-Regularized Leader (FoReL) algorithms, of which Multiplicative Weights Update (MWU) is perhaps the best known. MWU works by determining a probability distribution over actions by normalizing weights assigned to each action, with the weights equal to the exponential sum of past rewards and a learning rate. It satisfies the stronger non-negative regret property and therefore the no-absolute regret property.
Another algorithm we consider is Optimistic Multiplicative Weights Update (OMWU), which extends MWU with optimism by making the slight adjustment of counting the last value twice each iteration, a change which guarantees not just that the average policy is no-regret, but that the last one (the {\em last iterate}) is as well~\citep{daskalakis2018last}.
We also consider Regret Matching~\citep{hart2000simple} (RM) algorithms, which are the most widely used regret minimizers in CFR-based algorithms due to their simplicity and, unlike FoReL, lack of parameters. With RM, the policy distribution for iteration $t+1$ is selected for actions proportional to the accumulated positive regrets over iterations 0 to $t$. Regret Matching+ (RM+) is a variation that resets negative accumulated regret sums after each iteration to zero, and applies a linear weighing term to the contributions to the average strategy~\citep{tammelin2014solving}. The current state of the art algorithm, Discounted CFR (DCFR), is a parameterized algorithm generalizing RM+ where the accumulated positive and negative regrets are weighed separately as well the weight assigned to the contribution to the average strategy~\cite{brown2019solving}. The parameters used are $\alpha$ = 3/2, $\beta$ = 0 and $\gamma$ = 2, which are the values recommended by the authors.  All of these variants of RM are known to not have last iterate convergence in general and to not satisfy the non-negative regret property. (We do not know whether they satisfy the no-absolute-regret property.)


In addition to these standard no-regret algorithms, we introduce a new variant of RM called Regret Matching++ (RM++), which updates in a similar fashion to Regret Matching but clips the \textit{instantaneous} regrets at 0. That is, if $R^t(a)$ is the regret of action $a$ in round $t$ RM tracks $\sum_t R^t(a)$ while RM++ tracks the upper bound $\sum_t \max(R^t(a),0)$.\footnote{The same idea of clipping instantaneous regrets at 0 has recently been used by actor-critic approaches~\cite{srinivasan2018actor}.} Unlike all other RM variants used, however,  RM++ is not a no-regret algorithm. 



Lastly, we present results for the first two versions of LONR we analyzed theoretically: value-iteration style (LONR-V) and with asynchronous updates (LONR-A). For LONR-A, while the theory requires states be chosen for update uniformly at random, we instead run it on policy. (We add a small probability of a random action, 0.1, to ensure adequate exploration.)  Our results show that empirically this does not prevent convergence.  

The settings we use for our results are chosen to demonstrate LONR in settings where neither CFR nor standard RL algorithms are applicable.  For CFR, this means we choose settings with  repeated states and possibly a lack of terminals. For RL, this means considering settings with multiple agents.  Since our exposition of LONR is for a single agent setting, we now explain how we apply it in multi-agent settings. We use centralized training, so each agent has access to the current policy of the other agent. This allows the agent to update with the expected rewards and transition probabilities induced by the current policy of the other agent. 


\subsection{NoSDE Markov Game}

Our primary setting is a stateful one with multiple agents.  Such settings are naturally modelled as Markov games, a generalization of MDPs to multi-agent settings.
A \textbf{Markov Game} \begin{math} \Gamma \end{math} is a tuple \begin{math}(S, N, \mathbf{A}, T, R, \gamma) \end{math} where  \begin{math} S \end{math} is the set of states,  \begin{math} N = \{1, ..., n\} \end{math} is the set of players, the set of all state-action pairs \begin{math}\mathbf{A} = \bigcup_{s \in S}(\mathbf{\{s\}} \times \prod_{n \in N} \mathit{A_{n, s}})\end{math}, a transition kernel \begin{math} T : \mathbf{A} \mapsto \Delta(S) \end{math}, and a discount factor \begin{math} \gamma\end{math}.

Because Markov Games can model a wide variety of games, algorithms designed for the entirety of this class must be robust to particularly troublesome subclasses.  One early negative result found that there exist general-sum Markov Games in which no stationary deterministic equilibria exist, which \citet{zinkevich2006cyclic} term NoSDE games. These games have the property that there exists a unique stationary equilibrium with (randomized) policies where the Q-values for each agent are identical in equilibrium but their equilibrium strategies are not. Furthermore, additional complexity exists as the rewards of each player in this NoSDE game can be adjusted within a certain closed interval, where the resulting Q-values remain the same, but the stationary policy changes, thus making Q-value learning even more problematic.


\begin{figure}[t]
    \centering
\begin{subfigure}[b]{0.4\columnwidth}
\scalebox{0.4}{%
\begin{tikzpicture}[->,>=stealth',shorten >=1pt,auto,node distance=2.8cm,semithick]
  \tikzstyle{every state}=[]

  \node[state, minimum size=2cm] (A)                    {\Huge{$1$}};
  \node[state, minimum size=2cm]         (B) [right = 50mm of A] {\Huge{$2$}};
    \path (B) edge [bend left]             node {\huge{$R_{1}(2, SEND) = 0$}} (A)
          (A) edge [loop above]            node {\huge{$R_{1}(1, KEEP) = 1$}} (A)
          (B) edge [loop above]            node {\huge{$R_{1}(2, KEEP) = 3$}} (B)
          (A) edge [bend left]             node {\huge{$R_{1}(1, SEND) = 0$}} (B);
\end{tikzpicture}
}
\centering
\caption{Rewards for Player 1    \label{fig:my_labelR1}}
\end{subfigure}
\hfill 
\begin{subfigure}[b]{0.4\columnwidth}
\scalebox{0.4}{%
\begin{tikzpicture}[->,>=stealth',shorten >=1pt,auto,node distance=2.8cm,semithick]
  \tikzstyle{every state}=[]

  \node[state, minimum size=2cm] (A)                    {\Huge{$1$}};
  \node[state, minimum size=2cm]         (B) [right = 50mm of A] {\Huge{$2$}};
    \path (B) edge [bend left]             node {\huge{$R_{2}(2, SEND) = 0$}} (A)
          (A) edge [loop above]            node {\huge{$R_{2}(1, KEEP) = 0$}} (A)
          (B) edge [loop above]            node {\huge{$R_{2}(2, KEEP) = 1$}} (B)
          (A) edge [bend left]             node {\huge{$R_{2}(1, SEND) = 3$}} (B);
\end{tikzpicture}
}
\centering
\caption{Rewards for Player 2    \label{fig:my_labelR2}}
\end{subfigure}
\caption{NoSDE Markov Game \vspace{-5mm}}
\end{figure}
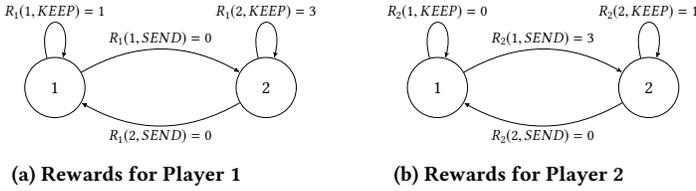

\begin{figure}[!ht]
   \centering
        \includegraphics[width = 0.99\columnwidth]{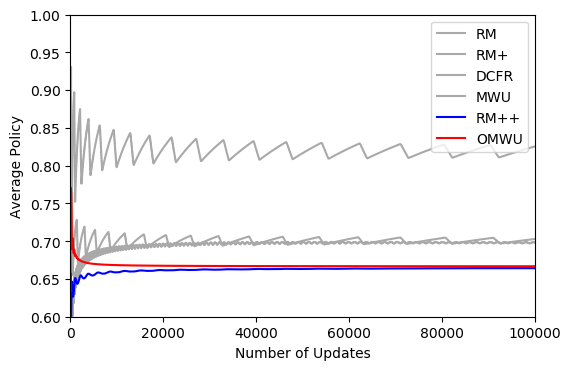}
        \caption{Average Policy for player1.  The two lowest lines are the first demonstration of convergence to stationary equilibrium in this setting.\vspace{-5mm}}
        \label{fig:V2Plots/LONR_V_NOSDE-SEND_ALL_AVG_1SEND}
\end{figure}

\begin{figure}[t]
    \centering
        \includegraphics[width = 0.99\columnwidth]{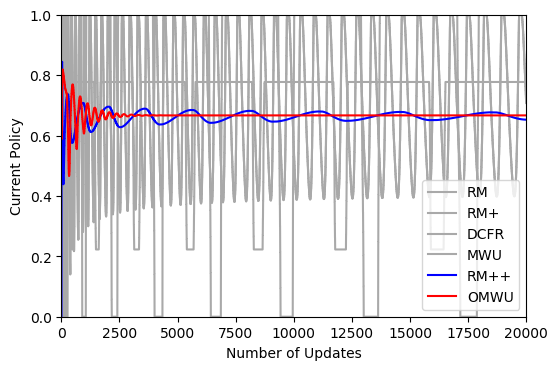}
        \caption{Last iterate for player1.  With RM++ or OMWU, LONR converges with the last iterate, not just on average.\vspace{-5mm}}
        \label{fig:V2Plots/LONR_V_NOSDE_ALL_SIX}
\end{figure}    

\begin{figure*}[htb]
    \centering
    \begin{subfigure}[b]{0.25\textwidth}
        \includegraphics[width = \textwidth]{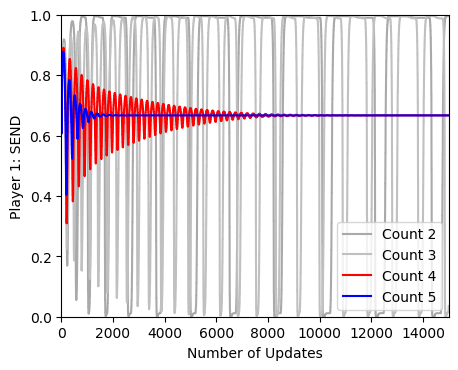}
        \caption{Impact of level of optimism on convergence}
        \label{fig:V2Plots/LONR_V_OMWU_COUNTS}
    \end{subfigure}
    \centering
    \begin{subfigure}[b]{0.25\textwidth}
        \includegraphics[width = \textwidth]{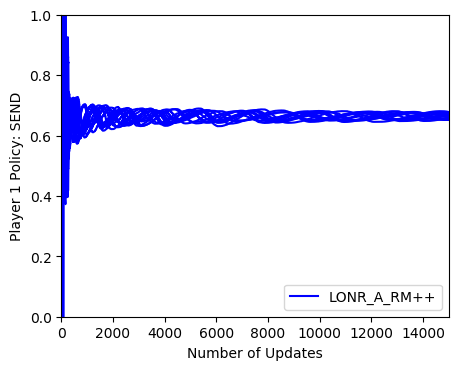}
        \caption{Multiple runs of LONR-A with RM++}
        \label{fig:V2Plots/LONR_A_NOSDE_RM++_SEND}
    \end{subfigure}
    \centering
    \begin{subfigure}[b]{0.25\textwidth}
        \includegraphics[width = \textwidth]{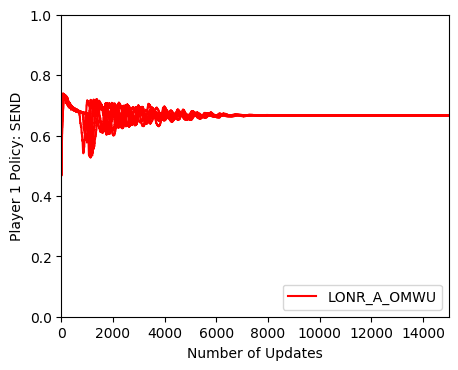}
        \caption{Multiple runs of LONR-A with OMWU}
        \label{fig:LONR_A_NOSDE_OMWU_SEND}
    \end{subfigure}
    \caption{Additional results for NoSDE Markov Game. (a) Optimism not only leads to convergence of the last iterate, but increasing optimism can affect convergence rates. (b and c) Asynchronous updates still converge using last iterate converging regret minimizers.\vspace{-5mm}}
\end{figure*}

The reward structure for the particular NoSDE game we use is shown in Figure~\ref{fig:my_labelR1} for Player 1 and Figure~\ref{fig:my_labelR2} for Player 2. Conceptually, a NoSDE game is a deterministic Markov Game with 2 players, 2 states, and each state has a single player with more than one action. The dynamics of a NoSDE game become cyclic as each player prefers to change actions when the other player does as well, which causes the non-stationarity.
In this instance, when player 1 sends, player 2 then prefers to send. This causes player 1 to prefer to keep, which in turn causes player 2 to prefer to keep. Player 1 then prefers to send and the cycle repeats.
Due to these negative results, Q-value learning algorithms cannot learn the stationary equilibrium.  The state of the art solution is still that of \citet{zinkevich2006cyclic} who give a multi-agent value iteration procedure which can approximate a cyclic (non-stationary) equilibrium.

No-regret algorithms are known to converge in self-play, but not necessarily to desirable points, e.g. Nash Equilibrium. This convergence guarantee is in the average policy. Our first results look at the average policies in the NoSDE game with LONR-V. Figure~\ref{fig:V2Plots/LONR_V_NOSDE-SEND_ALL_AVG_1SEND} show behavior of the average probability with which player 1 chooses to SEND.  The unique stationary equilibrium probability for this action is 2/3. Each algorithm shows convergence, but not to the same value.  Not shown but important is that each also is converging to the equilibrium $Q^*$ in the average Q values.

RM and MWU converge to a similar average policy (top two lines).  These two algorithms choose based on tracking the sum of regrets and rewards respectively.  RM+ and DCFR follow a similar path (next two lines), which makes sense given that RM+ is a special case of DCFR. RM++ and OMWU are the only two which find the stationary equilibrium policy (bottom two lines). These two are also the only two with last iterate convergence properties (OMWU provably and RM++ empirically). Figure~\ref{fig:V2Plots/LONR_V_NOSDE_ALL_SIX}, which plots the current iterate for each regret minimizer, shows that this holds in our NoSDE game as well.  RM++ and OMWU achieve last iterate convergence while for the other four cyclic behavior can be seen.\footnote{The figure shows last iterate convergence of the policy.  This also implies convergence of the value estimates.  See Appendix~\ref{sec:LIVE}.}
This result highlights NoSDE games as a setting where it would be interesting to theoretically study last iterate convergence in between simple normal form games~\cite{mertikopoulos2018cycles,bailey2018multiplicative} and rich, complex settings such as GANs~\citep{daskalakis2017training}.


While the theory behind OMWU states that the last value need only be counted twice, our results highlight the difference in the last iterate when more optimism is included (i.e. the last value is counted more than twice.) Specifically, in Figure~\ref{fig:V2Plots/LONR_V_OMWU_COUNTS} , we plot the last iterate for increasing counts of the last value. The figure indicates the role increased optimism plays in not only convergence versus divergence, but in how quickly convergence happens. In this case, despite the theory, counting twice does not lead to convergence in the last iterate, but 4 and above does. This simultaneously shows a negative and positive result: increased optimism is not known to work or be required in any other settings. Theoretically exploring this phenomenon is an interesting direction for future work.

Lastly, we analyze LONR-A, the asynchronous version of LONR. We restrict our results to the two which show last iterate convergence, RM++ (Figure~\ref{fig:V2Plots/LONR_A_NOSDE_RM++_SEND}) and OMWU (Figure~\ref{fig:LONR_A_NOSDE_OMWU_SEND}), plotting 100 runs of each.  They show that, despite a more natural process for choosing which state to update than our theory permits, we still see convergence.

\subsection{Additional Experiments}

Additional experiments which bridge the gap from MDPs to NoSDE Markov Games are presented in the Appendix. For a ``nicer'' Markov game than our deliberately challenging NoSDE game, we use the standard simple 2-player, zero-sum soccer game~\cite{littman1994markov}. With any of our six regret minimizers both LONR-V and LONR-A achieve approximate equilibrium payoffs on average.
For a setting to probe the assumptions of our theory in a setting closer to it, we run LONR on the typical benchmark GridWorld environment, an MDP.  Specifically we use the standard cliff-walking task which requires the agent to avoid a high-cost cliff to reach the exit terminal state.  Again, LONR-V and LONR-A learn the optimal policy (and optimal Q-values) despite regret minimizers that may not satisfy the no-absolute-regret property and, in the case of LONR-A, on policy state selection.


\begin{figure}[t]
    \centering
        \includegraphics[width = \columnwidth]{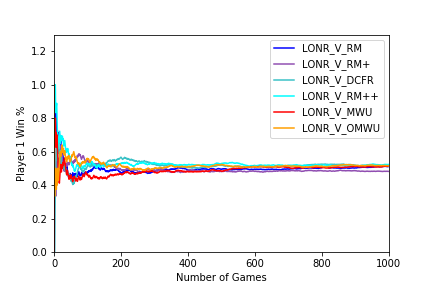}
        \caption{2-player, zero sum soccer game ~\cite{littman1994markov}: All tested no-regret algorithms combined with LONR reach equilibrium between opposing players in self-play.}
        \label{fig:V2Plots/LONR_V_SOCCER}
\end{figure}


\section{Conclusion}

We have proposed a new learning algorithm, local no-regret learning (LONR). We have shown its convergence for the basic case of MDPs (and limited extensions of them) and presented empirical results showing that it achieves convergence, and in some cases last iterate convergence, in a number of settings, most notably NoSDE games.  We view this as a proof-of-concept for achieving CFR-style results without requiring perfect recall or terminal states.  

Our results point to a number of interesting directions for future research.  First, a natural goal given our empirical results would be to extend our convergence results to Markov games.  Second, CFR also works in settings with partial observability by appropriately weighting the different states which correspond to the same observed history.  Third, we would like to relax the strong assumptions our results about asynchronous updates require.  All three  seem to rely on the same fundamental building block of better understanding the behavior of no-regret learners whose rewards are determined by (asynchronous) observations of other no-regret learners.
In particular, this leads to challenges due to the resulting non-stationarity of the transition kernels, which leads to hardness results that would need to be circumvented~\cite{yu2009online,yadkori2013online,radanovic2019learning}.
Some recent progress along these lines has been made~\cite{farina2018composability,kovavrik2018analysis}, but more work is needed.

Orthogonal directions are suggested by our empirical results about last iterate convergence.  Can we establish theoretical guarantees for NoSDEs or Markov games more broadly? Are there assumptions under which RM++ is no-regret, or is guaranteed to achieve last iterate convergence?  It empirically does in standard games like matching pennies and rock-paper-scissors which trip up most regret minimizers. If so does this represent a new style of algorithm to achieve last iterate convergence or is there a way to interpret its clipping of regrets as optimism?

\appendix

\section{Beyond MDPs}
\label{sec:beyond}

If we move beyond MDPs, $P$ and $r$ are no longer stationary and in general we have a $P_k$ and $R_k$.  This causes problems with the proof of Lemma~\ref{lem:regret}.  Recall the initial part of that proof, updated to this more general setting:
\begin{align*}
\overline{Q}_k(s,a)
&=  \frac{1}{k} \sum_{t = 1}^k Q_{t}(s,a)\\
&=  \frac{1}{k} \sum_{t = 0}^{k-1} r_{t}(s,a) + \gamma\mathbb{E}_{P_{t},\pi_{t}}[ Q_{t}(s',a')]\\
\end{align*}
In the original proof, we pulled the expectation over $P$ outside the sum, but now we cannot.  In particular, writing the expectation more explicitly gives
\begin{equation}
\frac{1}{k} \sum_{t = 0}^{k-1} r_{t}(s,a) + \gamma \sum_{s' \in \mathcal{S}} P_{t}(s' ~|~ s,a)\mathbb{E}_{\pi_{t}}[ Q_{t}(s',a')]
\end{equation}
We can still reverse the order of the sums, but the weighting terms now depend on $t$ so they cannot be moved outside.  More problematically, they also depend on $s$ and $a$, so it is not immediately clear how to generalize our results.
For intuition,
consider a state $s'$ where there are two actions.  At odd $k$, $r_k(s',a_1) = 1$ and $r_k(s',a_2) = 0$ and vice versa at even $k$.  It is a valid no-regret strategy to randomize uniformly over the actions, but if the $P_k$ are such that you only arrive in $s'$ from $s$ at odd $k$, then this gives an incorrect estimate.
In the remainder of this section, we analyze a special case where we can prove a variant of Lemma~\ref{lem:regret}.

\subsection{Time-invariant $P$}

If $P$ does not change with $k$, but $r$ does, we can still prove a version of Lemma~\ref{lem:regret}.  With a single state, this captures learning in normal-form games, where no-regret learning is indeed known to work.  This assumption is also common in the literature on ``online MDPs''~\citep{even2009online,mannor2003empirical,yu2009markov,ma2015online}
In this setting, a version of Lemma~\ref{lem:regret} can be proved, but now rather than having a constant operator $\mathcal{T}$ it now changes over time as
\begin{equation}
 \mathcal{T}_k Q(s,a) = \underbar{r}_k(s,a) + \gamma\mathbb{E}_P [\max_i Q(s',a_i)].
\end{equation}


\begin{lemma}
\begin{equation}
- \gamma \rho_{k-1}+  \mathcal{T}_k\underline{Q}_k(s,a) \leq\overline{Q}_k(s,a) \leq  \gamma \rho_{k-1}+  \mathcal{T}_k\underline{Q}_k(s,a).
\end{equation}
\end{lemma}
\begin{proof}
\begin{align*}
&\overline{Q}_k(s,a)\\
&=  \frac{1}{k} \sum_{t = 0}^{k-1} r_{t}(s,a) +\gamma\mathbb{E}_P [\frac{1}{k} \sum_{t = 0}^{k-1}\mathbb{E}_{\pi_{t}} [Q_{t}(s',a')]]\\
&\geq \frac{1}{k} \sum_{t = 0}^{k-1} r_{t}(s,a)+ \gamma \mathbb{E}_P [\max_i \frac{1}{k} \sum_{t = 0}^{k-1}Q_{t}(s',a_i) - \rho_{k-1}]\\
&= - \gamma \rho_{k-1} + \underbar{r}_k(s,a) + \gamma\mathbb{E}_P [\max_i \underline{Q}_k(s',a_i)]\\
&= - \gamma \rho_{k-1}+  \mathcal{T}_k\underline{Q}_k(s,a)\\
\end{align*}
As before, the key step is applying the no-regret property to obtain the inequality and we apply the same argument with the no-absolute-regret property to obtain the reverse inequality.
\end{proof}


\section{Omitted Proofs}
\label{sec:om-proofs}

\newtheorem*{lem.range}{Lemma~\ref{lem:range}}
\begin{lem.range}
Let $||r||_\infty = \max_{s,a} |r(s,a)|$.  Then $||Q_k - Q_0||_\infty \leq 1 / (1 - \gamma) ||r||_\infty + 2 ||Q_0||_\infty$
\end{lem.range}

\begin{proof}
By definition, $Q_{k}(s,a) = r(s,a) + \gamma\mathbb{E}_{P,\pi_k} [Q_{k-1}(s',a')]$.
Thus by the subadditive property of norms, $||Q_{k}||_\infty \leq ||r||_\infty + \gamma ||Q_{k-1}||_\infty$.
By induction, $||Q_{k}||_\infty \leq (\sum_{t = 0}^{k-1} \gamma^k)||r||_\infty + \gamma^k ||Q_0||_\infty$.
Thus $||Q_k - Q_0||_\infty \leq ||Q_k||_\infty + ||Q_0||_\infty \leq 1 / (1 - \gamma) ||r||_\infty + 2 ||Q_0||_\infty$.
\end{proof}

\newtheorem*{lem.contraction}{Lemma~\ref{lem:contraction}}
\begin{lem.contraction}
$||\underline{Q}_k - \mathcal{T}\underline{Q}_k||_\infty \leq \frac{1}{k}(1 / (1 - \gamma) ||r||_\infty + 2 ||Q_0||_\infty) + \gamma \rho_{k-1}$
\end{lem.contraction}

\begin{proof}
\begin{align*}
&||\underline{Q}_k - \mathcal{T}\underline{Q}_k||_\infty\\
&\leq ||\underline{Q}_k - \overline{Q}_k||_\infty + ||\overline{Q}_k - \mathcal{T}\underline{Q}_k||_\infty\\
&= ||\underline{Q}_k - \overline{Q}_k||_\infty+ \max_{s,a}|\overline{Q}_k(s,a) - \mathcal{T}\underline{Q}_k(s,a)|\\
&\leq ||\underline{Q}_k - \overline{Q}_k||_\infty + \gamma \rho_{k-1}\\
&= \frac{1}{k}||Q_k - Q_0||_\infty + \gamma \rho_{k-1}\\
&\leq \frac{1}{k}(1 / (1 - \gamma) ||r||_\infty + 2 ||Q_0||_\infty) + \gamma \rho_{k-1}\\
\end{align*}
The first step follows by the subadditive property of norms, the second by definition, the third by Lemma~\ref{lem:regret}, the fourth by definition, and the fifth by Lemma~\ref{lem:range}.
\end{proof}

\newtheorem*{lem.approximateFP}{Lemma~\ref{lem:approximateFP}}
\begin{lem.approximateFP}
Let $Q_0,Q_1,\ldots$ be a sequence such that $\lim_{k \rightarrow \infty} ||Q_k - \mathcal{T} Q_k||_\infty = 0$.  Then $\lim_{k \rightarrow \infty} Q_k = Q^*$.
\end{lem.approximateFP}

\begin{proof}
\begin{align*}
||Q_k - Q^*||_\infty
&\leq ||Q_k - \mathcal{T}Q_k||_\infty + ||\mathcal{T}Q_k - Q^*||_\infty\\
&= ||Q_k - \mathcal{T}Q_k||_\infty + ||\mathcal{T}Q_k -\mathcal{T} Q^*||_\infty\\
&\leq ||Q_k - \mathcal{T}Q_k||_\infty + \gamma||Q_k -Q^*||_\infty
\end{align*}
The first step follows by the subadditive property of norms, the second by optimality of $Q^*$, the third because $\mathcal{T}$ is a contraction map.  Rewriting yields
$$||Q_k - Q^*||_\infty \leq \frac{1}{1-\gamma}||Q_k - \mathcal{T}Q_k||_\infty$$

Thus, by assumption, $\limsup_{k \rightarrow \infty}||Q_k - Q^*||_\infty \leq 0$.  Since   $||Q_k - Q^*||_\infty \geq 0$, $\liminf_{k \rightarrow \infty}||Q_k - Q^*||_\infty \geq 0$.  Thus $\lim_{k \rightarrow \infty}||Q_k - Q^*||_\infty = 0$ and the result follows.
\end{proof}

\newtheorem*{lem.regret-async}{Lemma~\ref{lem:regret-async}}
\begin{lem.regret-async}
Let $s$ be the state selected uniformly at random and updated in iteration $t+1$, for which this is the $k$-th update and let $\overline{Q}_{t+1}(s,a) = 1/k \sum_{i = 1}^k Q_{t_i}(s,a)$ and $\overline{Q}_{t+1}(s',a) = \overline{Q}_{t}(s',a)$ for $s' \neq s$.
Then
\begin{align*}
&\min_{s'} \gamma (-\xi_{ss'}(k)- \rho_{k'}) +  \mathcal{T}\overline{Q}_t(s,a)\\ &\leq\overline{Q}_{t+1}(s,a) \leq  \max_{s'} \gamma (-\xi_{ss'}(k) + \rho_{k'}) +  \mathcal{T}\overline{Q}_t(s,a).
\end{align*}
\end{lem.regret-async}

\begin{proof}
By the definitions of LONR and no-regret algorithms,
\begin{align*}
&\overline{Q}_{t+1}(s,a)\\
&=  \frac{1}{k} \sum_{i = 1}^k Q_{t_i}(s,a)\\
&=  \frac{1}{k} \sum_{i = 1}^k r(s,a) + \gamma\mathbb{E}_{P,\pi_{t_i}}[ Q_{t_i}(s',a')]\\
&=  r(s,a) +\gamma\mathbb{E}_P [\frac{1}{k} \sum_{i = 1}^{k}\mathbb{E}_{\pi_{t_i}} [Q_{t_i}(s',a')]]\\
&=  r(s,a) +\gamma\mathbb{E}_P [-\xi_{ss'}(k) + \frac{1}{k'} \sum_{i = 1}^{k'}\mathbb{E}_{\pi_{\tau_i}} [Q_{\tau_i}(s',a')]]\\
&\geq  r(s,a) +\gamma\mathbb{E}_P [-\xi_{ss'}(k) + \max_{a'}\frac{1}{k'} \sum_{i = 1}^{k'}Q_{\tau_i}(s',a') - \rho_{k'}]\\
&\geq  \min_{s'} \gamma (-\xi_{ss'}(k)- \rho_{k'}) + r(s,a) +\gamma\mathbb{E}_P [\max_{a'}\frac{1}{k'} \sum_{i = 1}^{k'}Q_{\tau_i}(s',a')]\\
&=  \min_{s'} \gamma (-\xi_{ss'}(k)- \rho_{k'}) + r(s,a) +\gamma\mathbb{E}_P [\max_{a'}\overline{Q}_{t}(s',a')]\\
&=  \min_{s'} \gamma (-\xi_{ss'}(k)- \rho_{k'}) \mathcal{T}\overline{Q}_{t}(s,a')]\\
\end{align*}

This argument is essentially the same as in the proof of Lemma~\ref{lem:regret}, except that in the fourth equality we apply the definition of $\xi$ to yield a form to which we can then apply the no-regret property.  As before, the other half of the proof is symmetric and uses the no-absolute-regret property.
\end{proof}

\section{Empirical results for RM++}

Figure~\ref{fig:V2Plots/RPS_RM++} shows that the last iterate of RM++ converges to the equilibrium of rock-paper scissors.  Similar results, not shown, hold for matching pennies.  Prior work has shown that both RM and RM+ diverge in these games in terms of the last iterate (although they converge on average).  We also tested RM++ in Soccer and in Grid World. In both cases we achieved last iterate convergence.

\begin{figure}[t]
    \centering
        \includegraphics[width = 0.3\textwidth]{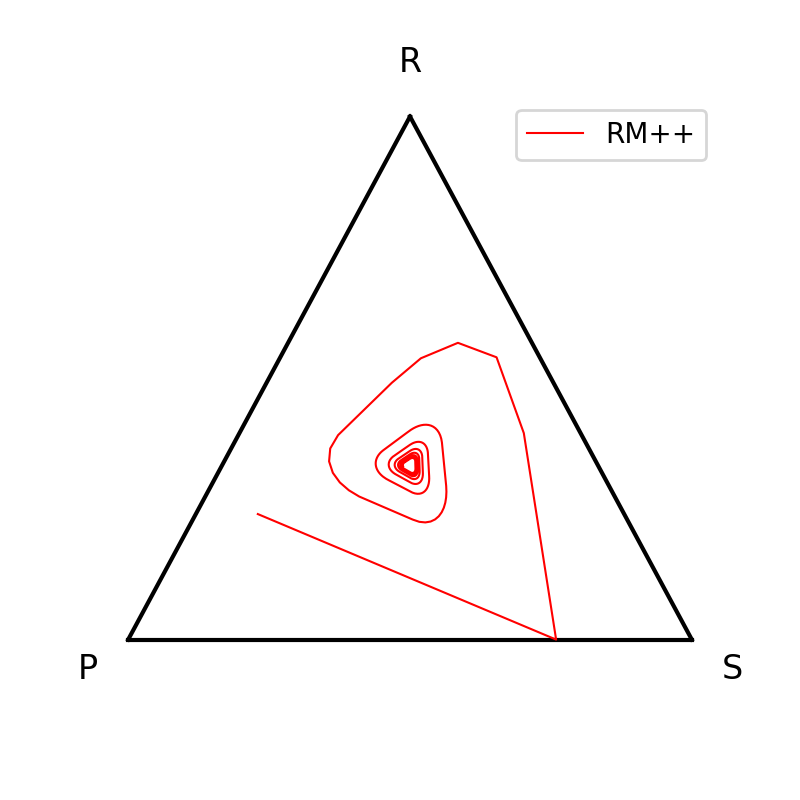}
    \caption{LONR-V policy with RM++ for Rock-Paper-Scissors       \label{fig:V2Plots/RPS_RM++} \vspace{-5mm}}
\end{figure}

\section{Last Iterate Convergence of Value Estimates} \label{sec:LIVE}

Our empirical convergence results show last iterate convergence for policies.  However, our theoretical results were about the convergence of the Q-value estimates.  At first glance this may appear an oversight, but a simple argument shows that last iterate convergence of the policies implies last iterate convergence of the Q-value estimates.  In particular, last iterate converenge means that $|\pi_k \cdot \underline{Q}_k(s) - \max_a \underline{Q}_k(s,a)| \leq \rho_k$, with $\lim_{k \rightarrow \infty} \rho_k = 0$.  By Theorem~\ref{thm:main}, $\lim_{k \rightarrow \infty} \underline{Q}_k = Q^*$.  Combining these shows that the $\pi_k$ are converging to $\pi^*$, which implies convergence of the Q-values.


\section{Experimental Details} 
\label{sec:om-experiments}

\subsection{LONR pseudocode}

In the following pseudocode, $N$ is the total number of agents, $n$ is current agent, and $S$ and  $s$ represent the total states and current state respectively. $A_n(s)$ denotes the set of actions for player $n$ in state $s$. $A_{-n}(s)$ denotes the set of actions of all other agents excluding agent $n$ in state $s$. $a$ refers to action of the current agent $n$ when unspecified. 
The policy update uses any no-regret algorithm. The update for Regret Matching++ is shown here.

\begin{algorithm}
\begin{algorithmic}[1]
\caption{LONR and Updates}\label{alg:lonr}

\Procedure{LONR-V}{$T, N, S, A_{n}$}\Comment{Value iteration}

    \State $\forall$ $n \in N, s \in S$, $a_{n} \in A_{n}(s)$ : 
    
    \State \quad $Q_{0}(n, s,a_{n}) \gets 0$, $\pi_{0}(n, s,a_{n}) \gets 0$ 
    \State $\quad$ RegretSums($n, s, a_{n}$) $\gets 0$, PolicySums($n, s, a_{n}$) $\gets 0$


    \State

    \For{$t$ from 0 to $T$}
        \State $\forall n \in N, s \in S$ : \State \quad Q-Update($n, s, t$)
        
        \State 
        \State $\forall n \in N, s \in S$ :
        \State \quad Policy-Update($n, s, t$)
        
        
    \EndFor
    \State

\EndProcedure

\Procedure{Q-Update}{$n, s, t$}\Comment{Update Q-Values}
        
        \For {each action $a_{n} \in A_{n}(s)$}
            \State $successors = getSuccessorStatesAndTransitionProbs(n, s, a_{n}, a_{-n})$
            \State $ActionValue \gets 0$
                
            \For {$s'$, $transProb$, $reward$ in $successors$}
            
                \State $nextStateValue \gets \sum_{a_{n}'} $
                $Q_{t}(n, s', a_{n}') \times \pi_{t}(n, s', a_{n}')$
            
                \State $ActionValue \gets ActionValue + transProb \cdot (reward + \gamma \cdot nextStateValue)$
                \EndFor
            \State $Q_{t+1}(n, s, a_{n}) \gets ActionValue$
        \EndFor
\EndProcedure


\State
\Procedure{Policy-Update}{$n, s, t$}\Comment{Regret Matching++}

        \State $ExpectedValue = \sum_{a_{n}} Q_{t+1}(n, s, a_{n}) \times \pi_{t}(n, s, a_{n})$

        
        \State
        
        \For{$a_{n} \in A_{n}(s)$}\Comment{RM++ Update Rule}
            \State $immediateRegret \gets max(0,  Q_{t+1}(n, s, a_{n}) - ExpectedValue$)
            \State RegretSums($n, s, a_{n}$) $\gets$ RegretSums($n, s, a_{n}$) + $immediateRegret$
        \EndFor
        \State
        
        \State $totalRegretSum = \sum_{i} $RegretSums($n, s, i$) 
        
        \State
        
        \For{$a_{n} \in A_{n}(s)$}\Comment{Update Policy}
            
            \If{$totalRegretSum > 0$}
                \State $\pi_{t+1}(n, s, a_{n}) = \frac{RegretSums(n, s, a_{n})}{totalRegretSum}$
            \Else
                \State $\pi_{t+1}(n, s, a_{n}) = \frac{1}{|A_{n}(s)|}$
            \EndIf
            \State
            \State PolicySums($n, s, a_{n}$) $\gets$ PolicySums($n, s, a_{n}$) $+ \pi_{t+1}(n, s, a_{n})$
            
        \EndFor
        
\EndProcedure

\end{algorithmic}
\end{algorithm}

\FloatBarrier

\subsection{LONR on GridWorld}

This task is a simple deterministic grid world MDP, in particular the cliff walking task used by \citep{van2009theoretical}, illustrated in Figure~\ref{fig:tasks-grid}.  As moves have a living cost of 1, we use $\gamma = 1$ (The optimal value from S is therefore -13.)  Because of the possibility of revisiting states, and receiving rewards/costs from non-terminals, CFR is not immediately applicable. Figure~\ref{fig:tasks-nosde} and Figure~\ref{fig:task-tiger} shows the results for both LONR-V and LONR-A in terms of the Q-value for the start state's optimal action of North.  LONR-V is deterministic with a single agent, so the plot represents a single run. For LONR-A we plot the results of 100 runs and their average. In both cases, we see convergence.

\begin{figure}[h]
    \centering
    \begin{subfigure}[b]{0.49\textwidth}
        \includegraphics[width = \textwidth]{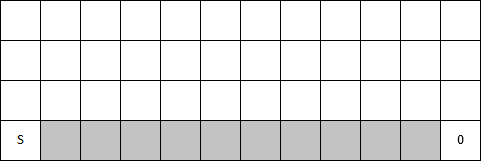}
        \caption{GridWorld: the agent's goal is to move from S to the 0 terminal.  Each move has a cost of 1 and the shaded states are terminals with a -100 payoff.  This particular grid world is the cliff walking task use by \citet{van2009theoretical} to evaluate Expected SARSA.}
        \label{fig:tasks-grid}
    \end{subfigure}
    \begin{subfigure}[b]{0.49\textwidth}
        \includegraphics[width = \textwidth]{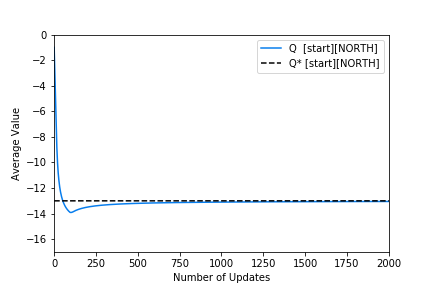}
        \caption{LONR-V on GridWorld}
        \label{fig:tasks-nosde}
    \end{subfigure}
    \begin{subfigure}[b]{.49\textwidth}
        \includegraphics[width = \textwidth]{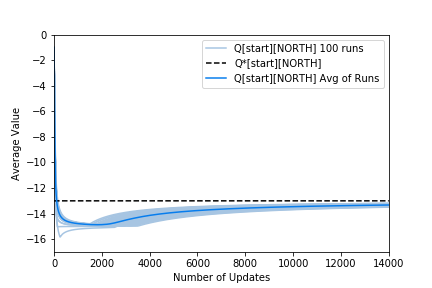}
        \caption{LONR-A on GridWorld}
        \label{fig:task-tiger}
    \end{subfigure}
    \caption{Tasks}
    \label{fig:tasks}
\end{figure}

\subsection{Soccer Game}

Here we analyze a simplified version of soccer, originally introduced in ~\citep{littman1994markov}, and subsequently widely used as an early benchmark game for multiagent Markov Games. Our implementation differs slightly in size, but maintains the general rules of the original game.

The soccer game is a two player, grid-style version based on real-life soccer. The size of the grid (field) is 2x4, where the first and last columns are the areas where each player can score. The 2x2 grid between the goal zones are cells in which the players can move. The game begins with each player set to a position on the grid, where one player has control of the ball. At each step of the game, the players each take an action, which are then executed jointly. The defending player is capable of stealing the ball by landing in the same cell as the player with the ball. If the player controlling the ball enters either of the goal cells, they receive 100 points and the other player receives -100 points, thus the game is zero-sum. For additional complexity, the order in which the actions are processed each iteration is randomized.

We run 2 agents against each other with LONR with each regret minimizer for 1000 games (a game runs until a player scores). We discount with $\gamma = 0.9$ to induce agents to score quickly. Before each new game, the position of the players (restricted to non-goal areas) is randomized, as well as who has initial control of the ball. The players then play 1000 games (with initial conditions again randomized) against each other using their learned policies (in this case, the average policy after training.) Figure~\ref{fig:V2Plots/LONR_V_SOCCER} shows the results of the trials. Each regret minimizer shows signs of convergence in self-play, as indicated by neither player dominating the other (each ends in the average as ties.)

\begin{figure}[t]
    \centering
        \includegraphics[width = \columnwidth]{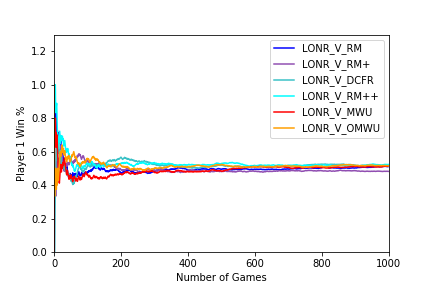}
        \caption{Soccer Game}
        \label{fig:V2Plots/LONR_V_SOCCER}
\end{figure}

  \bibliographystyle{ACM-Reference-Format}
  \bibliography{noregretq}
\end{document}